\documentclass[10pt,twocolumn,letterpaper]{article}

\usepackage{epigraph}

\usepackage{etoolbox}
\makeatletter
\patchcmd{\epigraph}{\@epitext{#1}}{\itshape\@epitext{#1}}{}{}
\makeatother

\usepackage{cvpr}
\usepackage{times}
\usepackage{epsfig}
\usepackage{graphicx}
\usepackage{amsmath}
\usepackage{amssymb}
\usepackage[amsthm,thmmarks]{ntheorem}
\usepackage{booktabs}

\usepackage{dsfont}

\newtheorem{lemma}{Lemma}

\newtheorem{proposition}{Proposition}

\usepackage[pagebackref=true,breaklinks=true,letterpaper=true,colorlinks,bookmarks=false]{hyperref}

\newcommand*\samethanks[1][\value{footnote}]{\footnotemark[#1]}

\cvprfinalcopy 

\def\cvprPaperID{2514} 

\ifcvprfinal\fi

\begin{document}

\title{Deep Learning under Privileged Information Using  Heteroscedastic Dropout}

\author{John Lambert\thanks{Equal contribution.} \\ 
  Stanford University
  \and 
  Ozan Sener\samethanks\hspace{1.5mm}\thanks{Work was completed while the author was affiliated with Stanford University.} \\
  Intel Labs
  \and 
  Silvio Savarese \\
    Stanford University
}

\maketitle

\begin{abstract}
Unlike machines, humans learn through rapid, abstract model-building. The role of a teacher is not simply to hammer home right or wrong answers, but rather to provide intuitive comments, comparisons, and explanations to a pupil. This is what the Learning Under Privileged Information (LUPI) paradigm endeavors to model by utilizing extra knowledge only available during training. We propose a new LUPI algorithm specifically designed for Convolutional Neural Networks (CNNs) and Recurrent Neural Networks (RNNs). We propose to use a heteroscedastic dropout (\ie dropout with a varying variance) and make the variance of the dropout a function of privileged information. Intuitively, this corresponds to using the privileged information to control the uncertainty of the model output. We perform experiments using CNNs and RNNs for the tasks of image classification and machine translation. Our method significantly increases the sample efficiency during learning, resulting in higher accuracy with a large margin when the number of training examples is limited. We also theoretically justify the gains in  sample efficiency by providing a generalization error bound decreasing with $\mathcal{O}(\frac{1}{n})$, where $n$ is the number of training examples, in an oracle case.
\end{abstract}
\vspace{-5mm}
\section{Introduction}
\epigraph{``Better than a thousand days of diligent study is one day with a great teacher."}{--- \textup{Japanese Proverb}}
\vspace{-3mm}
It is a common belief that human students require far fewer training examples than any learning machine \cite{Vapnik2009LUPI}. No
doubt this has to do with the fact that effective teachers provide much more than the correct answer to their pupils;
they provide an explanation in addition to the result.

In a typical machine learning setup, we present tuples $\{(x_i,y_i)\}_{i=1}^{n}$ to a machine learning
model. One way to introduce an ``explanation'' to a supervised learning system would be to provide some sort of privileged information, which we entitle $x^\star$. In practice, one can incorporate the triplets $\{(x_i,x^\star_i, y_i)\}_{i=1}^{n}
$ into a learning system at training time and in the continue to make use of only $x$ in the testing stage, without any access to $x^\star$. In other words, the ``Student'' has access to privileged information while interacting with the
``Teacher'' during training, but in the test stage the ``Student'' operates without the supervision of the ``Teacher''.
This paradigm is called Learning Under Privileged Information (LUPI) and was introduced by Vapnik and Vashist \cite{Vapnik2009LUPI}.

Vapnik and Vashist \cite{Vapnik2009LUPI} provide a LUPI algorithm for Support Vector Machines (SVMs). From an
algorithmic perspective, the privileged information is utilized to estimate slack values of the SVM constraints. From a theoretical perspective, this algorithm accelerates the rate at which the upper bound on error drops from 
$\mathcal{O}\left(\sqrt{\frac{1}{n}}\right)$ to a far steeper curve of $\mathcal{O}\left(\frac{1}{n}\right)$, where $n$ is the number of required samples.

\begin{figure}[t]
\includegraphics[width=\columnwidth]{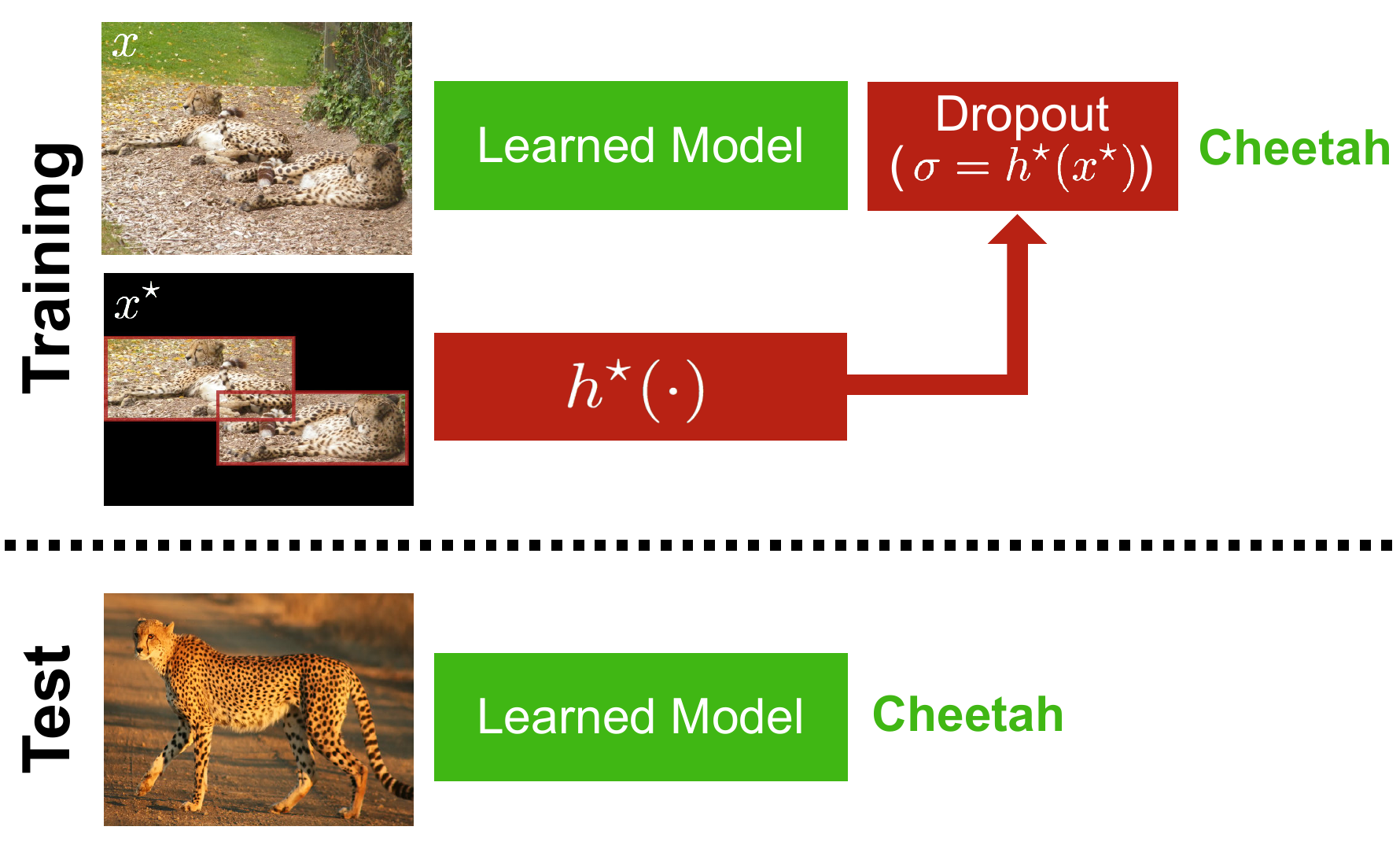}
\caption{In the Learning Under Privileged Information (LUPI) paradigm, a teacher provides additional information during training. In this work, we propose to utilize this information in order to control the variance of the Dropout. Since the Dropout's variance is not constant, we call this a \emph{Heteroscedastic Dropout}. Our empirical and theoretical analysis suggests that \emph{Heteroscedastic Dropout} significantly increases the sample efficiency of both CNNs and RNNs, resulting in higher accuracy with much less data.}
\vspace{-5mm}
\end{figure}

Privileged information is ubiquitous: it usually exists for almost any machine learning problem. However, we do not see
wide adoption of such methods in practice. The major obstacle is the fact that the original LUPI framework proposed in \cite{Vapnik2009LUPI} is only valid for
SVM-based methods. Indeed, many have shown that the privileged information can be introduced into the loss function under
a multi-task or a distillation loss in an algorithm-agnostic way. However, we raise the question, could it and \emph{should} it be fed in as an input instead of an additional task? If so, how would we go about doing so in an algorithm-agnostic way?

We define a new class of LUPI algorithms by making a structural specification. We consider a hypothesis class
such that each hypothesis is a combination of two functions -- namely, a deterministic function taking $x$ as an input, and a
stochastic function taking $x^\star$ as an input. When $x^\star$ is not available in the
test stage, the ``Student'' simply makes a Bayes optimal decision and marginalizes the model over $x^\star$. Our
structural specification makes this marginalization straightforward while not compromising the expressiveness of the model. This structure is natural in the context of Convolutional Neural Networks (CNNs) and Recurrent Neural Networks (RNNs) thanks to dropout. Dropout is a widely adopted tool to regularize neural networks by multiplying the activations of a neural network at some layer with a random vector. We simply extend the dropout to \emph{heteroscedastic dropout} by making its variance a function of the privileged information. In other words, dropout becomes the stochastic function taking $x^\star$ as an input and marginalizing the function corresponds to not utilizing dropout in the test phase. In order to be able to train the heteroscedastic dropout, we use Gaussian dropout instead of Bernoulli because the key technical tool we use is the re-parameterization trick \cite{KingmaWelling2014}
 which is only available for some specific distributions, including the Gaussian.
 
The rationale behind heteroscedastic dropout follows the close relationship between Bayesian learning and dropout presented by Gal and Gharamani\cite{gal_bayes}. Dropout can be considered a tool to approximate the uncertainty of the output of a neural network. In our proposed heteroscedastic dropout, the privileged information is used to estimate this uncertainty so that hard examples and easy examples are treated accordingly during training. Our theoretical study suggests that the accurate computation of a model's uncertainty can accelerate the rate at which a CNN's upper bound on error drops, from the typical rate of $\mathcal{O}\left(\sqrt{\frac{1}{N}}\right)$ to a faster $\mathcal{O}(\frac{1}{n})$, where $n$ is the number of training examples. In an oracle case for a dataset with $600K$ training examples, this theoretical upper bound would allow us to learn a model with identical generalization error with $\sqrt{6\times 10^8}\approx 775$ samples instead of $600$K and is thus hugely significant. Although the practical gain we observe is nowhere close, it is still very significant. 

We evaluate our method in experiments with both CNNs and RNNs, and show a
significant accuracy improvement over two canonical problems, image classification and multi-modal machine translation. 
As privileged information, we offer a bounding box for image classification and an image of the scene described in a sentence for machine translation. 
Our method is problem- and modality-agnostic and can be incorporated as long as dropout can be utilized in the original problem and the privileged information can be encoded with an appropriate neural network.
\vspace{-2mm}
\section{Related Work}
The key aspects that differentiate our work from the literature are:
$i)$ our method is applicable to any deep learning architecture which can utilize dropout, $ii)$ we do not use a multi-task or distillation loss, $iii)$ we provide theoretical justification suggesting higher sample efficiency, $iv)$ we perform experiments for both CNNs and RNNs. A thorough review of the related literature is provided below.

\vspace{2mm}
\noindent\textbf{Learning Under Privileged Information:}
Learning under Privileged Information (LUPI) is initially proposed by Vapnik and Vashist \cite{Vapnik2009LUPI, JMLRVapnik2015b}. It extends the Support Vector Machine (SVM) by empirically estimating the slack values via privileged information. 
This method is further applied to various computer vision problems \cite{vision_paper1, vision_paper2,NIPS2014ObjectLocalizationSSVM} as well as ranking \cite{ranking}, clustering \cite{clustering} and metric learning \cite{metric_learning} problems. These method are based on max-margin learning and are not applicable to CNNs or RNNs.

One closely related work is \cite{gplupi}, extending Gaussian processes to the LUPI paradigm. Hern\'{a}ndez-Lobato \emph{et al.} \cite{gplupi} use privileged information to estimate the variance of the noise in their model. Similarly, we use the privileged information to control the variance of the dropout in CNN and RNN models. However, their method only applies to Gaussian processes, whereas we target neural networks. 

\vspace{2mm}
\noindent \textbf{Learning CNNs Under Privileged Information:}
The LUPI paradigm has also been studied recently in the context of CNNs. In contrast to max-margin methods, the literature on learning CNNs under privileged information heavily uses the distillation framework, following the close relationship between distillation and LUPI studied in \cite{LopezPaz2016ICLRUnifyingDistilledPrivileged}.

Hoffman \emph{et al.} demonstrated a multi-modal distillation approach to incorporating an additional modality as side information
\cite{Hoffman_2016_CVPR_ModalityHallucination}. They start with a pre-trained network and distill the information from the privileged network to a main neural network in an end-to-end fashion.

Multi-task learning is a straightforward approach to incorporate privileged information. However, it does not necessarily satisfy a no-harm guarantee (\ie privileged information can harm the learning). More importantly, the no-harm guarantee will very likely be violated since estimating the privileged information (\ie solving the additional task) might be \emph{even more challenging} than the original problem.

When the privileged information is binary and shares the same spatial structure as the original data, such as is the case with segmentation occupancy or bounding box information, it can also directly be incorporated into the training of CNNs by masking the activations. Group Orthogonal neural networks \cite{GroupOrthogonalCNN} follow this approach. However, this approach is limited to very specific class of problems. 

The loss value of a CNN can be viewed as analogous to the SVM slack variables. Following this analogy, Yang \emph{et al.} \cite{MIMLFCNCVPR2017} use two networks: one for the original task, and one for estimating the loss using the privileged information. Learning occurs through parameter sharing between them. 

Our method is different from aforementioned works since we do not use either a distillation or a multi-task loss.

\vspace{2mm}
\noindent\textbf{Learning Language under Privileged Visual Information:}
Using images as privileged information to learn language is not new. Chrupala et al.\cite{imaginet} used a multi-task loss while learning word embeddings under privileged visual information. The embeddings are trained for the task of predicting the next word, as well the representation of the image. Analysis of this model \cite{imaginet, kadar} suggests that the embeddings learned by using vision as a privileged information are significantly different than language only ones and correlate better with human judgments. Recently, Elliott \emph{et al.} \cite{multi30k} collected a dataset of images with English captions as well as German translations of captions. Using this dataset, a neural machine translation \emph{under privileged information} model is developed following the multi-task setup \cite{ImaginationElliottAkos17}. 

\vspace{2mm}
\noindent\textbf{Dropout and its Variants:} Dropout is a well studied regularization technique for training deep networks. To-the-best of our knowledge, we are the first to \emph{specifically utilize privileged information} to control the variance of a dropout function. Here, we summarize the existing methods which control the variance of the dropout using variational inference or information theoretical tools. Although these tools have never been applied to the LUPI paradigm, we utilize some of the technical tools developed in these works.

We use multiplicative Gaussian dropout instead of Bernoulli dropout. Gaussian dropout is first introduced in \cite{JMLR_Srivastava14a_Dropout}. Its variational extension \cite{variationaldropout} uses local re-parameterization to perform Bayesian learning.

The Information Bottleneck (IB) \cite{ib} is a powerful framework which can enforce various structural assumptions. The IB framework has been applied to CNNs and RNNs using stochastic gradient variational Bayes and the re-parametrization trick \cite{KingmaWelling2014}. Perhaps closest to our method, Achille and Soatto \cite{informationdropout} use the information bottleneck principle to learn disentangled representations when a CNN with Gaussian Dropout is used.  The authors introduce many
ideas upon which we build; specifically, our hypothesis class (Eqn. 4) is very similar to the architecture they propose.
The main architectural difference is their choice to define the variance as a function of $x$, whereas we make it a function of $x^{\star}$. We also use similar distributional priors and a similar training procedure. On the other hand, we apply these ideas to a completely different problem with a different theoretical analysis.  The information bottleneck has been applied to LUPI for SVMs \cite{iblupi}. However, this method does not apply to neural networks. 

Although we use IB \cite{TishbyZaslavsky2015InfoBottleneck}, Gaussian dropout \cite{JMLR_Srivastava14a_Dropout} and the re-parametrization trick \cite{KingmaWelling2014}, we are the first to our knowledge to apply any of these methods to the LUPI problem.
\section{Method} 
Consider a machine learning problem defined over a
compact space $\mathcal{X}$ and a label space $\mathcal{Y}$. We also consider a loss function $l(\cdot,\cdot)$ which
compares a prediction with a ground truth label. In learning under privileged information, we also have additional
information for each data point defined over a space $\mathcal{X}^\star$, which is only available during the training. In
other words, we have access to i.i.d. samples from the data distribution as $x_i,x_i^\star,y_i \sim p(x,x^\star,y)$ during
training. However, in test we will only be given $x \sim p(x)$. Formally, given a function class
$h(\cdot;\mathbf{w})$ parameterized by $\mathbf{w}$ and data $\{x_i,x_i^\star,y_i\}_{i \in [n]}$, a typical aim is to solve the following optimization problem;
\vspace{-1mm}
\begin{equation} 
\min_{\mathbf{w}} E_{x,y \sim p(x,y)}[l(y,h(x;\mathbf{w}))] 
\end{equation}

We propose to do so by learning a multi-view model using both $x$ and $x^\star$
and to use the marginalized model in test when $x^\star$ is not available. Consider a parametric function class for the
multi-view data $h^+:\mathcal{X}\times\mathcal{X}^\star \rightarrow \mathcal{Y}$. The training problem becomes:
\begin{equation} \vspace{-2mm}
    \min_{\mathbf{w}} E_{x,x^\star,y \sim p(x,x^\star,y)} [l(y,h^+(x,x^\star;\mathbf{w}))] 
\end{equation}
This is equivalent to a classical supervised learning problem defined over a space
$\mathcal{X}\times\mathcal{X}^\star$ and any existing method like CNNs can be used. In order to solve the inference
problem, we consider the following marginalization
\vspace{-1mm}
\begin{equation}
    h(x;\mathbf{w}) \equiv E_{x^\star \sim p(x^\star|x)} [h^+(x,x^\star;\mathbf{w})]
    \label{eq:exp} \vspace{-2mm}
\end{equation}

The major problem in this formulation is the intractability of this expectation, as $p(x^\star|x)$ is unknown. We propose to restrict the class of
functions in a way that the expectation is straightforward to compute. The form we propose is a parametric family such
that the privileged information controls the variance, whereas the main information (\ie
information available in both training and test) controls the mean. The specific form we use is:
\begin{equation} \vspace{-1mm} \label{eq:our}
h^+(x,x^\star;\mathbf{w}) = h^o(x;\mathbf{w^o}) \odot \mathcal{N}(\mathbf{1}, h^\star(x^\star;\mathbf{w^\star})) 
\end{equation}
where $\odot$ represents the Hadamard product and the stochastic function $\mathcal{N}(\mathbf{1}, h^\star(x^\star; \mathbf{w^\star}))$ is a normal random variable with a constant mean function and a covariance function parameterized by
$x^\star$ and $\mathbf{w^\star}$. We also decompose $\mathbf{w}$ as two disjoint vectors as $\mathbf{w}=[\mathbf{w^o},\mathbf{w^\star}]$. Moreover, in this formulation, the expectation
defined in (3) becomes straightforward and can be shown to be $h(x;\mathbf{w}) = h^o(x;\mathbf{w^o})$. We visualize this structural specification in Figure~\ref{fig:fstar}. 
\begin{figure}[ht]
    \includegraphics[width=\columnwidth]{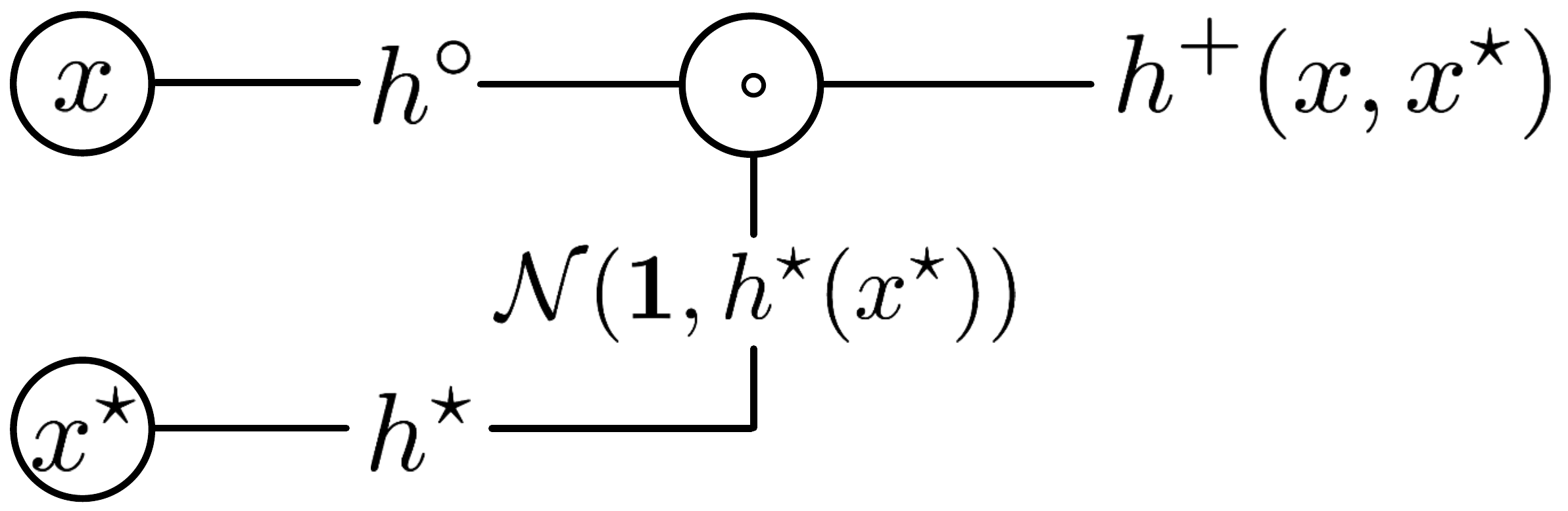}
    \caption{The structure we propose. Privileged information is only used for estimation of the variance of the heteroscedastic dropout.}
    \label{fig:fstar}
    \vspace{-5mm}
\end{figure}

We use neural networks to represent $h^o$ and $h^\star$ and learn their parameters using the information bottleneck. Since the output space is discrete (we address classification), we denote the representation of the data as $h(x;\mathbf{w})$ and compute the output as $\texttt{softmax}(h(x;\mathbf{w}))$. We explain the details of training in the following sections.
\vspace{-1mm}
\subsection{Information Bottleneck for Learning} 
We need to control the role of $x^\star$ in LUPI. The information bottleneck has already been used for this purpose \cite{iblupi}; however, we do not need this explicit specification because our structural specification directly controls the role of $x^\star$. We use the information bottleneck for a rather different reason, its original reason, learning a minimal and sufficient joint representation of $x, x^\star$ which captures all the information about $y$. This is similar to \cite{informationdropout}, and we use the same log-Normal assumption. The Lagrangian of the information bottleneck can be written as \emph{(see \cite{TishbyZaslavsky2015InfoBottleneck} for details)};
\begin{equation}
        \mathcal{L} = H(y|z) + \beta I(x,x^\star;z)
    \end{equation}
where $z$ is the joint representation of $x,x^\star$ computed as \mbox{$z=h^+(x,x^\star;\mathbf{w})$}. These terms can be computed as;
{\small
\begin{equation}
\begin{aligned}
    I(x^\star,x;z) &= E_{x,x^\star \sim p(x,x^\star)}[KL(p_w(z|x,x^\star) || p_w(z))] \\ \\
    H(y|z) &\simeq E_{x,x^\star,y \sim p}[E_{z\sim p_w(z|x,x^\star)}[-\mbox{log }p_w(y|z)]] 
\end{aligned}
\end{equation}}

\noindent where $p_w(\cdot)$ represents the distributions computed over our model with parameters $\mathbf{w}$. In order to compute the KL divergence, we need an assumption about the prior over representations $p(z)$. As suggested by \cite{informationdropout}, the log-Normal distribution follows the empirical distribution $p(z)$ when ReLu is used. Hence, we use the log-Normal distribution and compute the KL divergence as \emph{(see appendix for full derivation)}; 
\begin{equation}
    KL(p_w(z|x,x^\star) || p_w(z)) \sim \|\log h^\star(x^\star;\mathbf{w^\star})\| .
\end{equation}
    
Combining them, the final optimization problem is;
\begin{equation}
    \min_w \frac{1}{n}\sum_{i=1}^n E_{z\sim p_w(z|x,x^\star)} [\mbox{log }p(y_i|z)] + \beta \|\log h^\star(x^\star_i;w^\star)\|
\end{equation}
This minimization is simply the cross-entropy loss with regularization over the logarithm of the computed variances of the heteroscedastic dropout, and can be performed via the re-parametrization trick in practice when $h^o$ and $h^\star$ are defined as neural networks. 
We further justify the choice of IB regularization via experimental observation: without it, optimization leads to NaN loss values. We discuss the details of the re-parametrization trick in the following sections.

\subsection{Implementation}

In this section, we discuss the practical implementation details of our framework, specifically pertaining to image classification with CNNs and machine translation with RNNs. For the classification setup, we
use the image as $x$, object localization information as $x^\star$, and image label as $y$. For the translation setup, we
use the sentence in the source language as $x$, an image which is the realization of the sentence as $x^\star$, and the
sentence in the target language as $y$. 

We make a sequence of architectural decisions in order to design $h^\circ$ and $h^\star$. For the classification problem, we design both of them as CNNs and share the convolutional layers. The inputs are $x$, an image, and $x^\star$, an image with a blacked-out background. We use the VGG-Network \cite{Simonyan2014} as an architecture and simply replace each dropout with our form of heteroscedastic dropout. We show the details of the architecture with the re-parameterization trick in Figure~\ref{fig:cnn}. We also normalize images with the ImageNet pixel mean and variance. As data augmentation, we horizontally flip images from left to right and make random crops. 

We use a two-layered LSTM architecture with 500 units as our RNN cell and use heteroscedastic dropout between layers of LSTMs. The main reason behind this choice is the fact that dropout in general has only been shown to be useful for connections between LSTM
layers. We use attention \cite{attention} and feed the image as a feature vector computed using the VGG\cite{Simonyan2014} architecture pre-trained on ImageNet. We give the details of the LSTM with re-parametrization trick in Figure~\ref{fig:lstm}. For the inference, we use beam search over 12 hypotheses. Our LSTM implementation directly follows the baseline implementation provided by OpenNMT \cite{opennmtpy}.

\begin{figure*}[ht]
\begin{center}
    \includegraphics[width=\linewidth]{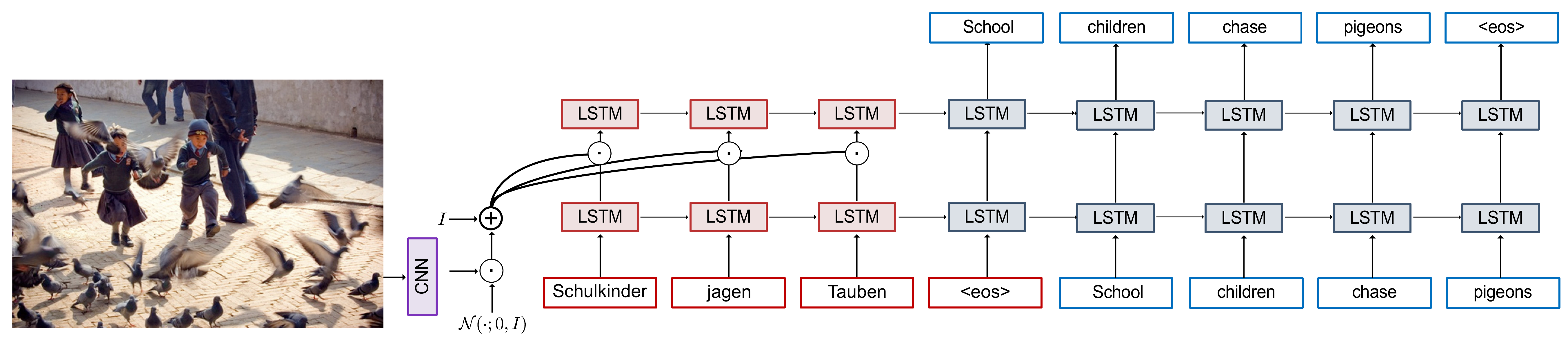}
    \caption{\textbf{Multi-Modal Machine Translation} We show the LSTM architecture we use, which incorporates the re-parameterization trick and heteroscedastic dropout connections. We use dropout only between layers and share among cells following \cite{gal_rnn}. We do not use any dropout in inference since the image is not available during test.}
    \label{fig:lstm}
    \vspace{-9mm}
\end{center}
\end{figure*}

\begin{figure}[ht]
\begin{center}
    \includegraphics[width=\columnwidth]{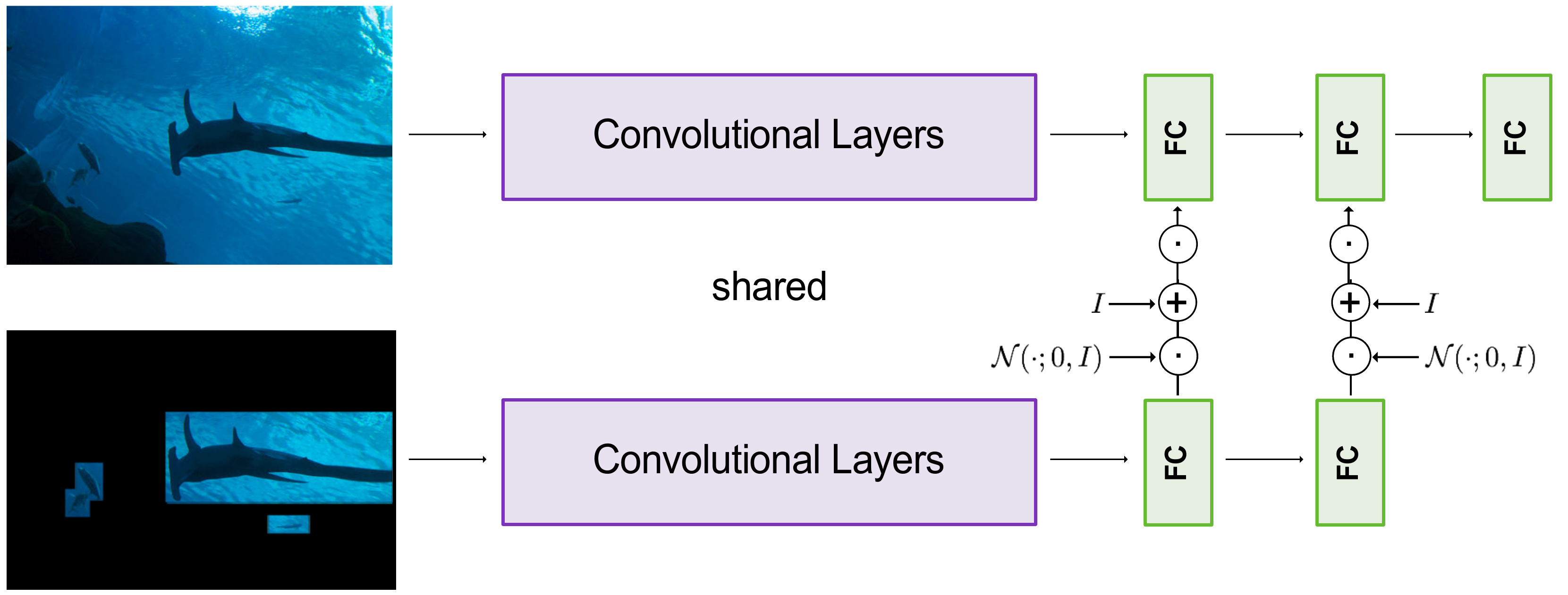}
\end{center}
\vspace{-3mm}
\caption{\textbf{Image Classification} We show the CNN architecture we used in our experiments, along with the re-parameterization trick and heteroscedastic dropout connections. We do not use any dropout in inference since localization bounding boxes are not available during test.}
\vspace{-6mm}
\label{fig:cnn}
\end{figure}

\vspace{2mm}
\noindent \textbf{Hyperparameter Settings}
We use a standard learning rate across all image classification experiments, setting our initial learning rate to $1.0 \times 10^{-3}$, and
tolerating 5 epochs of non-increasing validation set accuracy before decaying the learning rate by $10$x. For multi-modal machine translation, we use an initial learning rate of $1.0 \times 10^{-3}$ and halve the learning rate every epoch after the 8$^{th}$ epoch. We use the ADAM \cite{KingmaBa2015} optimizer in PyTorch for both image classification and multi-modal machine translation. All CNN weights are initialized according to the method of He \textit{et al.} \cite{ResNetHe2016CVPR} and a decay of $1\times 10^{-4}$ was used for image classification. For multi-modal machine translation, we do not use any weight decay and initialize weights according to \cite{opennmtpy}.  
\vspace{-2mm}
\section{Experimental Results}
In order to evaluate our method, we perform various experiments using both CNNs and LSTMs. We test our method with CNNs for the task of image classification and with LSTMs for the task of machine translation. 
In the rest of this section, we discuss the baselines against which we compare our algorithm and the datasets we use.

\vspace{2mm}
\noindent \textbf{Datasets}: We perform our experiments using the following datasets; \textbf{ImageNet \cite{imagenet_cvpr09}:} A dataset of 1.3 million labelled images, spanning 1000 categories. We only use the subset of 600 thousand images which include localization information. 
\textbf{Multi-30K\cite{multi30k}:} A dataset of 30 thousand Flickr images which are captioned in both English and German. We use this dataset for multi-modal machine translation experiments. In Multi-30K, whereas the English captions are directly annotated for images, the German captions are only translations of the English captions. Hence, during the ground truth translation, the images were privileged information never seen by the translators. This property makes this dataset a perfect benchmark for LUPI.

\noindent \textbf{Baselines:} We compare our method against the following baselines. \textbf{No-$x^\star$:} a baseline model not using any privileged information. \textbf{Gaussian Dropout \cite{JMLR_Srivastava14a_Dropout}:} A multiplicative Gaussian dropout with a fixed variance. \textbf{Multi-Task:} We perform multi-task learning as a tool to utilize privileged information. We compare both regression to bounding box coordinates and denote it as Multi-Task w/ B.Box, as well as direct estimation of the RGB mask and denote it as Multi-Task w/ Mask. We use this self-baseline only for CNNs since there are many published multi-task methods for machine translation with multi-modal information and we compare with them all. In addition to these self-baselines, we also compare with the following published work: \textbf{GoCNN\cite{GroupOrthogonalCNN}:} a method for CNNs with segmentation as privileged information which proposes to mask convolutional weights with segmentation masks. \textbf{Information Dropout\cite{informationdropout}:} a regularization method that utilizes injection of multiplicative noise in the activations of a deep neural network (but as a function of the input $x$, not $x^{\star}$). \textbf{MIML-FCN\cite{MIMLFCNCVPR2017}:} a CNN-based LUPI framework designed for multi-instance problems. Our problem is not multiple instance; however, we still make a comparison for the sake of completeness. \textbf{Modality Hallucination\cite{Hoffman_2016_CVPR_ModalityHallucination}:} Distillation-based LUPI method designed for multi-modal CNNs. \textbf{Imagination \cite{ImaginationElliottAkos17}:} Distillation-based LUPI method designed for multi-modal machine translation \emph{(see appendix for implementation details)}.

\subsection{Effectiveness of Our Method}

We compare our method with the No-$x^\star$ baseline for image classification using the ImageNet dataset. We perform experiments by varying the number of training examples logarithmically. This is key since the main motivation behind our LUPI method is \emph{learning with less data} rather than having higher accuracy. We report several results in Table~\ref{tab:imagenet_1000_logarithmic_data_sizes} and visualize additional data points in Figure ~\ref{fig:smoothcurvesxstarvsnoxstar}.

\begin{table}[htb]
\centering
\caption{ Classification Test Accuracy on 1000 ILSVRC Classes. Because the ILSVRC server prohibits large numbers of test submissions, which we required to evaluate at different sizes of sample data, we use a hold-out set of $50K$ images from ImageNet CLS-LOC 
as our test set. The authors of \cite{Simonyan2014} report a $7.4\%$ Multi-Crop, top-5 error rate when training on $\sim 1.3$M images. Where we report ``No-$x^{\star}$,'' we describe the results of a classical CNN learning method. All 1-crop evaluations below were carried out with a center crop. All $25$K models diverged.}
\vspace{-5mm}
\label{tab:imagenet_1000_logarithmic_data_sizes}
\begin{center}
\begin{tabular}{ lccccc} 
\toprule 
& \multicolumn{4}{c}{Number of Training Images} \\
Model &  25K  & 75K &  200K  & 600K  \\ 
\midrule 
Single Crop top-1 & & & & & \\
No-$x^\star$ & - & 37.85 & $\mathbf{55.99}$ & $66.66$  \\
Our LUPI & - & $\mathbf{42.30}$ & $55.51$ & $\mathbf{66.77}$  \\ 
\midrule 
Single Crop top-5 & & & &  \\
No-$x^\star$ & - & $62.76$ & $\mathbf{79.21} $  & $\mathbf{86.90}$ \\
Our LUPI & - & $\mathbf{67.13}$ & $78.89$ & $86.88$  \\ 
\midrule 
Multi-Crop top-1 & & & &  \\
No-$x^\star$ & - & 39.99 & $\mathbf{58.7}$& $\mathbf{69.20}$  \\
Our LUPI & - & $\mathbf{44.95}$ & $58.41$ & $69.10$ \\ 
\midrule 
Multi-Crop top-5 & & & &  \\
No-$x^\star$ & - & 64.49 & 81.0 & $88.60$ \\
Our LUPI & - & $\mathbf{69.19}$ &$\mathbf{81.15}$ & $\mathbf{88.64}$  \\ 
\bottomrule
\end{tabular}
\end{center}
\vspace{-5mm}
\end{table}

\begin{figure}[ht]
\begin{center}
    \includegraphics[width=\columnwidth]{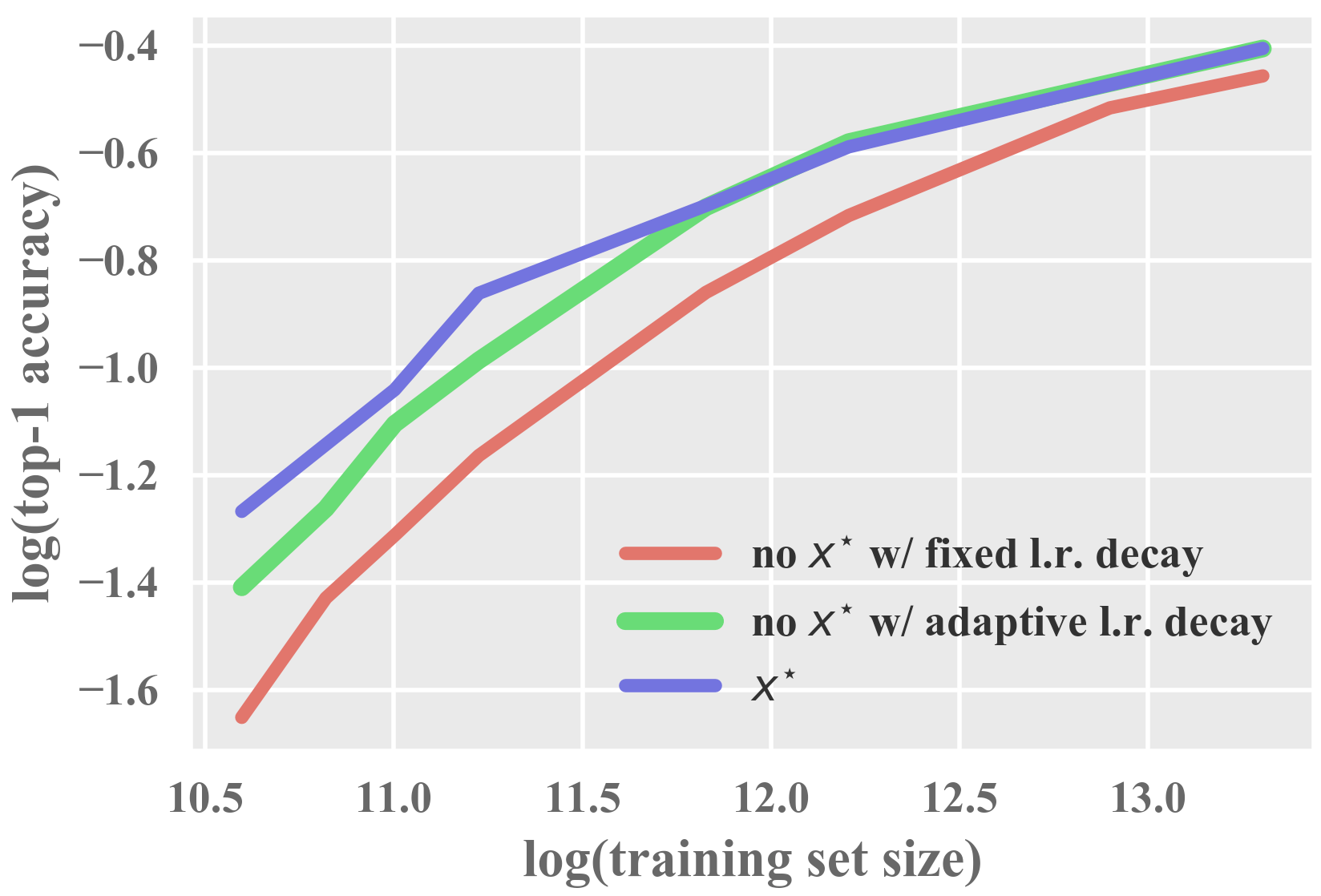}
\end{center}
\vspace{-3mm}
\caption{Accuracy vs. training set size for ImageNet classification. Each data point denotes a VGG-16 network trained with batch normalization. The accuracies of models trained \textit{with} $x^{\star}$ are depicted in \textbf{blue}; those trained \textit{without} are depicted in \textbf{green} (via an adaptive learning rate decay schedule) and \textbf{red} (via a fixed learning rate decay schedule). \textbf{Adaptively modifying the learning rate according to performance on a hold-out set yields massive gains in low- and mid-scale data regimes} when compared with decaying the learning rate at fixed intervals, \textit{e.g.} every 30 training epochs. }
\vspace{-7mm}
\label{fig:smoothcurvesxstarvsnoxstar}
\end{figure}

Our method is quite effective for a training dataset size of $75K$ images; however, it has no positive impact for the $200K$ and $600K$ cases. Even more importantly, the smaller the training set, the larger the improvement. For $75K$ images, it results in $5\%$ single crop top-1 accuracy improvement; whereas, for $200K$, it matches the performance. This simply suggests that our algorithm is particularly effective for low- and mid-scale datasets. This result is quite intuitive since with increasing dataset size, all algorithms can effectively learn and reach an optimal accuracy which is possible under the model class. Hence, the role of an ``intelligent teacher'' is to provide privileged information to learn with less data. In other words, LUPI is not a way to gain extra accuracy regardless of the dataset; rather, \emph{it is a way to significantly increase the data efficiency}. We do not perform a similar experiment for machine translation since the available dataset is a mid-scale and our LUPI method demonstrates asymptotic accuracy increases at full dataset. 

\subsection{Data Efficiency of Our Method and Baselines}
In order to compare the data efficiency gain of our method against baselines, we perform image classification and multi-modal machine translation experiments. We use $75K$ ImageNet images since our main goal is identify insights regarding data sample efficiency gains and using a smaller training set makes this analysis possible. We summarize the image classification experiments in Table~\ref{tab:75kbaselines_imagenet_1000_classification} and multi-modal machine translation experiments in Table 3. Our method outperforms all baselines for both tasks, for image classification with a significant margin. 

\begin{table}[htb]
\caption{ We compare our method's performance with several baselines. We train with 75 Images per each of the 1000 ImageNet classes, leaving us with  75 $\times 10^3$ images in total. We outperform each model and are competitive with GoCNN, a model specifically designed for the problem of learning with segmentation data using various architectural decisions. Evaluation is carried out on the held-out set of images from our holdout test set.}
\label{tab:75kbaselines_imagenet_1000_classification}
\centering
\resizebox{\columnwidth}{!}{%
\begin{tabular}{lcccc}
\toprule
 & \multicolumn{2}{c}{Single Crop} &  \multicolumn{2}{c}{Multi-Crop}\\
\textbf{Model} & top-1 & top-5 & top-1 &top-5 \\
\midrule
No-$x^\star$ \cite{Simonyan2014} & 37.85 &
62.76 & 39.99 & 64.49 \\
\midrule
MIML-FCN \cite{MIMLFCNCVPR2017}/ResNet & 35.61 & 59.66 & 38.3 & 62.3 \\
Modal. Hallucination \cite{Hoffman_2016_CVPR_ModalityHallucination}& 37.66 & 63.15 & 40.45 & 65.95 \\
Info. Dropout \cite{informationdropout} & 38.09 & 63.52 & 41.84 & 67.47 \\
Gaussian Dropout \cite{JMLR_Srivastava14a_Dropout} & 38.80 & 63.64 & 41.0 & 65.3  \\
MIML-FCN \cite{MIMLFCNCVPR2017}/VGG& 39.54 &  64.43 & 42.0 & 66.4 \\
Multi-Task w/ Bbox & 39.96 &  64.79 & 42.4 & 66.6 \\
Multi-Task w/ Mask \cite{LearningDeconvolutionNetwork2015}  & 40.48 & 65.62 & 43.18 & 67.68\\
GoCNN & 41.43 &  66.78 & 44.5 & $\mathbf{69.3}$\\
Our LUPI  & $\mathbf{42.30}$ & $\mathbf{67.13}$ & $\mathbf{44.95}$ & 
$69.19$ \\
\bottomrule
\end{tabular}}
\end{table}
\begin{table}[htb] 
\label{tab:machine_translation}
\caption{We compare our method for multi-modal machine translation with several baselines. We report BLEU\cite{bleu, moses} and METEOR\cite{meteor} metrics. Some baselines only report English(en)$\rightarrow$German(de) results, and exclude de$\rightarrow$en.} \vspace{-3mm}
\begin{center}
\resizebox{\columnwidth}{!}{%
\begin{tabular}{ lcccc} 
\toprule 
& \multicolumn{2}{c}{en$\rightarrow$de} & \multicolumn{2}{c}{de$\rightarrow$en}  \\
Model &  BLEU  & Meteor &  BLEU  & Meteor \\ 
\midrule 
No $x^{\star}$ (following \cite{opennmtpy}) & $35.5$ & $54.0$ & $40.19$ & $55.8$ \\
\midrule
\emph{Multi-Modal} & & & & \\
Toyama et al. \cite{toyama} & $36.5$ & $56.0$ & $-$ & $-$ \\
Hitschler et al. \cite{hitschler} & $34.3$ & $56.1$ & $-$ & $-$ \\
Calixto et al. \cite{calixto1} & $36.5$ & $55.0$ & $-$ & $-$ \\
Calixto et al. \cite{calixto2} & $37.3$ & $55.1$ & $-$ & $-$ \\
\midrule
\emph{LUPI} & & & & \\
Imagination \cite{ImaginationElliottAkos17} & $36.8$ & $55.8$ & $40.5$ & $56.0$ \\
Ours  & $\mathbf{38.4}$  & $\mathbf{56.9}$ & $\mathbf{42.4}$ & $\mathbf{57.1}$ \\
\bottomrule
\end{tabular}}
\end{center}
\vspace{-6mm}
\end{table} 
\noindent\textbf{Image Classification with Privileged Localization} 

One interesting result is that our network $h^{\star}$ is clearly learning much more than to predict a random constant for its output $\Sigma$, the covariance matrix used for reparameterization; in fact, our network outperforms the network that produces a $\Sigma$ whose entries are drawn from pure Gaussian noise by $>3.8\%$. We analyze our method for CNNs both theoretically and qualitatively in Section~\ref{sec:analysis} and conclude that our method learns to control the uncertainty of the model and results in an order of magnitude higher data efficiency, explaining this large margin.

Furthermore, GoCNN \cite{GroupOrthogonalCNN}, an architecture specifically designed for the problem of learning with segmentation data using various architectural decisions, results in a significant accuracy improvement competitive with our method in a small dataset regime. However, GoCNN's  performance relative to other baselines begins to degrade at a dataset size of 200K images, leading to a top-1 accuracy decrease
of $-5.26\%$ in comparison with Bernoulli dropout and $-4.47\%$ with our heteroscedastic dropout method. This is an intuitive result because GoCNN's rigid architectural decisions inject significant bias into the model.

 Because Information Dropout relies upon sampling from a log-normal distribution with varying variance, it is heteroscedastic. However, compounded Information Dropout layers which exponentiate samples from a normal distribution lead to unbounded activations; thus, a suitable squashing function like the sigmoid must be employed to bound the activations. We find its performance can actually decrease accuracy when compared with a ReLU nonlinearity.

\noindent\textbf{Multi-modal Machine Translation} Our method results in a significant accuracy improvement measure by both BLEU and METEOR scores. One interesting observation is that our method outperforms various multi-modal methods which use image information in both training and test. This counterintuitive result is due to the way in which the dataset is collected. The Multi-30k\cite{multi30k} dataset is collected by simply translating the English captions of $30K$ images into German. A LUPI model \cite{ImaginationElliottAkos17} was already shown to perform better than multi-modal translation models which can use both images and sentences at test time \cite{calixto1,calixto2,hitschler,toyama}. This surprising result is largely due to the fact that the translators did not see the images while providing ground truth translations. 
More importantly, the effectiveness of visual information in machine translation in a privileged setting is also intuitive following the results of \cite{imaginet}. Chrupala et al.\cite{imaginet} show that when image information is used as privileged information in the learning of word representations, the quality of such representations increases. Hence, a multi-modal paradigm for learning language (e.g. with privileged visual information) and vice versa is a fruitful direction for both natural language processing and computer vision communities and our method performs quite effectively on this task.

In summary, our results overperform all baselines for both multi-modal machine translation and image classification experiments using both CNNs and RNNs. These results suggest that our method is effective and generic.

\vspace{-2mm}
\subsection{Learning under Partial Privileged Information}
Although privileged information naturally exists for many problems, it is typically not available for all points. Thus, it is common to encounter a scenario in which the entire training data is labelled; however, only a small portion includes privileged information. In other words, we typically have a dataset which is the union of $\{x_i,y_i\}_{i\in[n]}$ and $\{x_j,x^\star_j,y_j\}_{j \in [m]}$ where $m \ll n$. In order to experiment with this setting, we vary the amount of $x^\star$ available. We present the result in Figure~\ref{fig:varyxstar} for machine translation and in the appendix for image classification.

The results in Figure~\ref{fig:varyxstar} suggest that even when only a small portion ($2\%$ for machine translation, $4\%$ for image classification) of the data has privileged information, our method is effective resulting in a significant accuracy increase very similar to the one we obtained with $100$\% of privileged information. 

\begin{figure}
\includegraphics[width=\columnwidth]{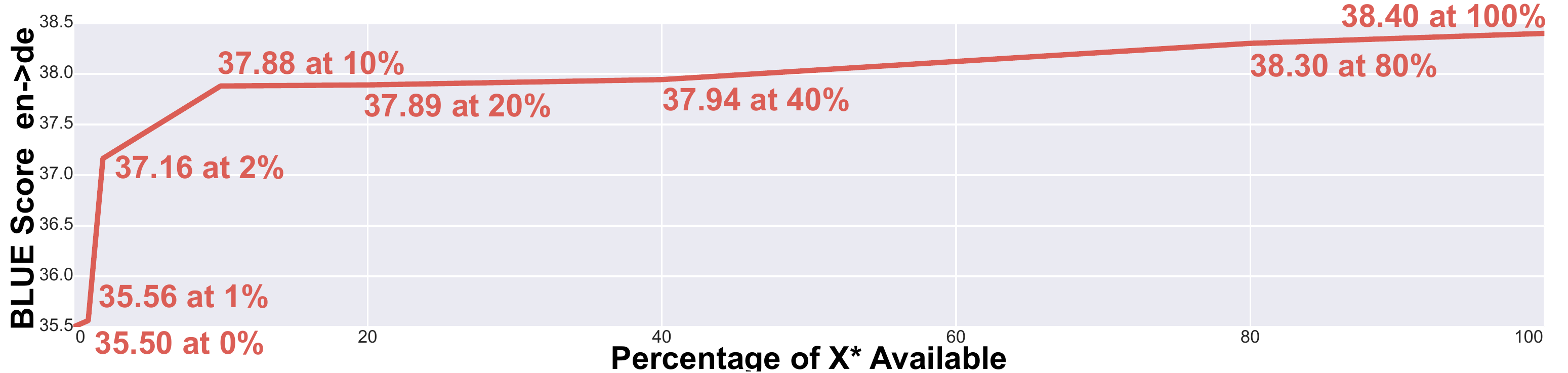}
\caption{Accuracy vs. $x^\star \%$ for multi-modal machine translation. An identical experiment on image classification is shown in appendix.}
\label{fig:varyxstar}
\vspace{-5mm}
\end{figure}

\vspace{-2mm}
\section{Analysis of the Algorithm}
\label{sec:analysis}
Our empirical analysis suggests a strong data-efficiency increase when privileged information is incorporated using our method. It is interesting to quantify this increase in terms of the theoretical learning rate. For the case of SVMs, Vapnik \emph{et al.} \cite{Vapnik2009LUPI} showed that utilizing privileged information can result in a generalization error bound with rate $\mathcal{O}(\frac{1}{n})$ instead of $\mathcal{O}(\sqrt{\frac{1}{n}})$ where $n$ is the dataset size. Our experimental results suggests a similar story for CNNs, but the theoretical justification can not be extended from \cite{Vapnik2009LUPI} since their analysis is specific to SVMs. In this section, we endeavor to answer this question for our algorithm. We show that our method is capable of converting an $\mathcal{O}(\sqrt{\frac{1}{n}})$ error rate (derived in Proposition 1) into $\mathcal{O}(\frac{1}{n})$ in an oracle setting for CNNs. We rigorously prove that it is possible to reach an $\mathcal{O}(\frac{1}{n})$ rate using our structural assumptions; however, we do not provide any argument for the optimization landscape. In other words, our results are only valid with an oracle optimizer which can find the solution satisfying our assumptions. The study of the loss function and the optimization remains an open problem; however, we provide strong empirical evidence that using SGD with information bottleneck regularization enables faster learning.

We start by presenting a bound over the generalization error of CNNs with no privileged information. This result directly follows from \cite{robust}, and we include it here for the sake of completeness. Loss functions of CNNs based on $l_2$ distance are Lipschitz continuous when the non-linearity is the rectified linear unit, the pooling operation is max-pooling and the softmax function is used to convert activations into logits. Moreover, any learning algorithm with a Lipschitz loss function admits the following result \cite{robust};
\vspace{-2mm}
\begin{proposition}[{{\cite[Example 4]{robust}}}]
Given $n$ i.i.d. samples drawn from $p(x,y)$ as $\{\mathbf{x}_i,y_i\}_{i\in[n]}$, if a loss function $l(y,h(x;\mathbf{w}))$ is $\lambda^l$-Lipschitz continuous function of $x$ for all $y, \mathbf{w}$, bounded by $L$ and $\mathcal{X}\times\mathcal{Y}$ has a covering number $N_{\epsilon}(\mathcal{X},|\cdot|_2)=K$, then with probability at least $1-\delta$, \vspace{-3mm}
\[ \vspace{-5mm} \begin{aligned}
\left|E_{\mathbf{x},y \sim p(x,y)}[l(y, h(\mathbf{x};w))] - \frac{1}{n}\sum_{i\in[n]}l(y_i, h(\mathbf{x}_i;w))\right| 
\\ \leq  \lambda^l \epsilon + L \sqrt{\frac{2K\log 2 + 2\log (1/\delta)}{n}}. \vspace{-1mm}
\end{aligned} \] 
\end{proposition}
This proposition simply details the baseline $\mathcal{O}(\sqrt{\frac{1}{n}})$ error rate under no privileged information. In order to intuitively explain how our algorithm can accelerate this learning to $\mathcal{O}(\frac{1}{n})$, consider the following oracle algorithm. Using the privileged information, one can estimate the uncertainty (variance) of the neural network and can use the inverse of this estimate as a the variance of the heteroscedastic dropout. Since the heteroscedastic dropout is multiplicative, this results in unit variance regardless of the input. In a similar fashion, this oracle algorithm can bound the variance with an arbitrary constant. Following this oracle algorithm, we show that when the variance is properly controlled, our method can reach an $\mathcal{O}(\frac{1}{n})$ rate. Consider the population distribution of number images per class versus the empirical distribution as $\epsilon_y = E[y^\intercal y] - \frac{1}{n}\sum_i y_i^\intercal y_i$. The value $\epsilon_y$ is purely a property of the way in which a was dataset collected and must be treated independently of the learning. Hence, we do not study the rate at it vanishes. We present the following proposition and defer its proof to the appendix:
\vspace{-2mm}
\begin{proposition}
Given $n$ i.i.d. samples drawn from $p(x,y)$ as $\{\mathbf{x}_i,y_i\}_{i\in[n]}$ and a loss function defined as $\|h(x;w)-y\|_2^2$ where $h(\cdot;w)$ is a CNN, assume that any path between input and output has maximum weight $M_w$, the total number of paths between input and output is $P$, and for all training points $x_i$, $h(x_i;w) \leq M_z$ and $Var(h(x_i;w)) \leq \xi^2$. With probability at least $1-\delta$, \vspace{-2mm}
\[ \vspace{-2mm}\begin{aligned}\vspace{-2mm}
&\left|E_{\mathbf{x},y \sim p(x,y)}[l(y, h(\mathbf{x};w))] - \frac{1}{n}\sum_{i\in[n]} l(y_i, h(\mathbf{x}_i;w))\right| \\ \leq  
 &\frac{2C\left(\left(\xi+1\right)\log\frac{1}{\delta} +M_w \left(3\xi+M_z\right) \log\frac{P}{\delta}\right)}{3n} + \left(2C+1\right)\epsilon_y.
\end{aligned} \]
where $\xi \leq \delta$.
\label{mainthm} \vspace{-2mm}
\end{proposition}
This proposition means that learning with sample efficiency $\mathcal{O}(\frac{1}{n})$ is indeed possible as long as one can bound the variance of the output($\xi$) with an arbitrary number $\delta$. Hence, full control of the output variance makes learning with higher sample efficiency possible. A question remains: \emph{is it possible to learn this oracle solution by using SGD with information bottleneck regularization?} Unfortunately, we have no theoretical answer for this question and leave it as an open problem. However, we study this problem empirically and show that there is a strong empirical evidence suggesting that the answer is affirmative.
\vspace{-2mm}
\begin{figure}[ht]
\includegraphics[width=\columnwidth]{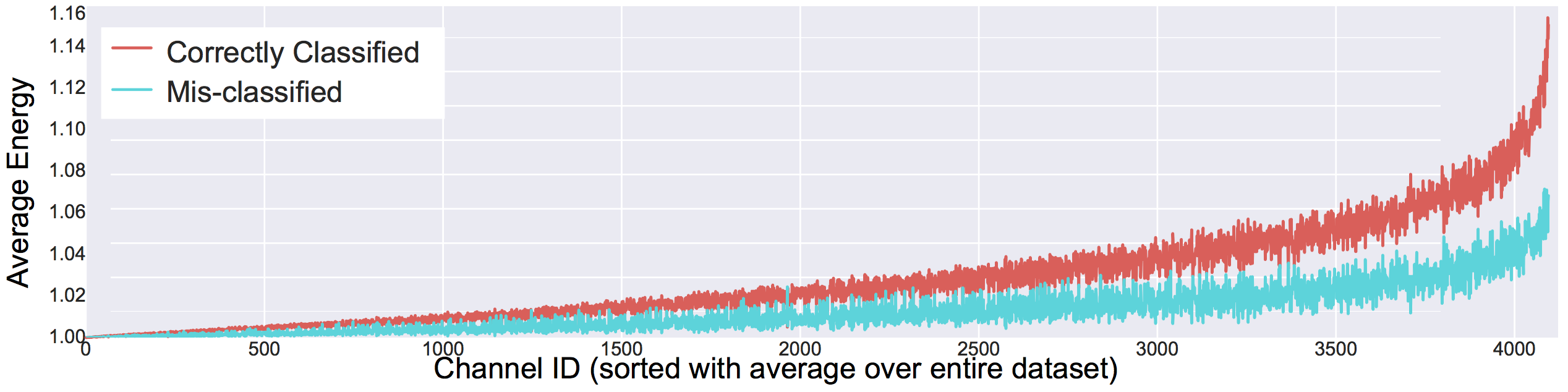}
\caption{For 8000 random samples from the validation set that our heteroscedastic dropout algorithm mis-classifies, as well as 8000 random samples it correctly classifies, we plot the average of activations per dimension (we sort the 4096 dimensions in terms of average energy over the full dataset for clarity). }
\label{fig:qual_var}
\vspace{-4mm}
\end{figure}

A realistic estimate of variance is typically not possible without a strong parametric assumption; however, we can use the simple heuristic that samples from the validation set that our algorithm mis-classifies should have higher variance than the samples which are correctly classified.  We plot the average energy of computed dropout variances per fully connected neuron for mis-classified and correctly classified examples in Figure~\ref{fig:qual_var}. Interestingly, our method consistently assigns larger multiplier (dropout) values for correctly classified samples and significantly smaller values for mis-classified samples. This strongly supports our hypothesis since when the low heteroscedastic dropout is multiplied with the high-variance mis-classified examples, their final variance will be low, possibly bounded by the $\sigma_0$.
\vspace{-3mm}
\section{Conclusion}
\vspace{-1mm}
We described a learning under privileged information framework for CNNs and RNNs. We proposed a heteroscedastic dropout formulation by making the variance of the dropout a function of privileged information.

Our experiments on image classification and machine translation suggest that our method significantly increases the sample efficiency of both CNNs and LSTMs. We further provide an upper bound over the generalization error of CNNs suggesting a sample efficient learning (with rate $\mathcal{O}(\frac{1}{n})$) in the oracle case when privileged information is available. We make our learned models as well as the source code available \footnote{http://svl.stanford.edu/projects/heteroscedastic-dropout}\footnote{https://github.com/johnwlambert/dlupi-heteroscedastic-dropout}.

\vspace{-3mm}
\section{Acknowledgements}
\vspace{-1mm}
We thank Alessandro Achille for his help on comparison with information dropout. We acknowledge the support of Toyota  (1191689-1-UDAWF), MURI (1186514-1-TBCJE);  Panasonic (1192707-1-GWMSX).

{\small
\bibliographystyle{ieee}
\bibliography{dlupi}
}

\appendix


\section{Proof of Proposition 1}
The proof of proposition 1 is available at \cite{robust} as Example 4. However, we include here a simpler proof for the sake of completeness.
\begin{proof}
\begin{small}
We will start with
\[
\begin{aligned}
&\left|E_{x,y \sim p_\mathcal{Z}}[l(y,h(x;w)] - \frac{1}{n}\sum_{i \in [n]} l(y_i,h(x_i;w)\right| \\
&\overset{(a)}{\leq} \left|\sum_{j \in [K]} E[l(y,h(x;w)| (x,y) \in C_j] \mu_{j} \right. \\ &\left. \quad \quad -  \sum_{j \in [K]} E[l(y,h(x;w)| (x,y) \in C_j] \frac{|n_j|}{n} \right| \\
 &\quad +  \left| \sum_{j \in [K]} E[l(y,h(x;w)| (x,y) \in C_j] \frac{|n_j|}{n}  \right. \\ &\quad \quad \left.-\frac{1}{n}\sum_{i \in [n]} l(y_i, h(x_i;w) \right| \\
  &\overset{(b)}{\leq}\left|\sum_{j \in [K]} E[l(y, h(x;w)| (x,y) \in C_j] (\mu_{j} -   \frac{|n_j|}{n}) \right|
 \\ &\quad +\frac{1}{n} \left|\sum_{j \in [K]} \sum_{i \in n_j} E[l(y,h(x;w)| (x,y) \in C_j]  - l(y_i,h(x_i;w))\right| \\
   &\overset{(c)}{\leq} \left|\sum_{j \in [K]} E[l(y,h(x;w)|z \in C_j] (\mu_{j} -   \frac{|n_j|}{n})\right| + \lambda^l \epsilon  \\
 \end{aligned}
\]
\end{small}

In $(a)$, we use the fact that the space has an $\epsilon$-cover; and denote the cover as $\{C_j\}_{j \in [K]}$ such that each $C_j$ has diameter at most $\epsilon$. We further define an auxiliary variable $\mu_j=p((x,y) \in C_j)$ and $n_j = \sum_i \mathds{1}[(x_i,y_i) \in C_j]$ and used the triangle inequality. In $(b)$, we use $i \in n_j$ to represent $(x_i,y_i) \in C_j$. Finally, in $(c)$ we use the fact that each ball has diameter at most $\epsilon$ and the loss function is $\lambda^l$-Lipschitz.

We can bound $E[l(x,y)|z \in C_j]$ with a maximum loss $L$ and use the Breteganolle-Huber-Carol inequality (\emph{cf} Proposition A6.6 of \cite{wellner}) in order to bound $\sum_{j} \mu_{j} -   \frac{|n_j|}{n}$. 

Combining all, we observe that with probability at least $1-\delta$,
\[
\begin{aligned}
&\left|E_{x,y \sim p_\mathcal{Z}}[l(y,h(x;w))] - \frac{1}{n}\sum_{i \in [n]} l(y_i,h(x_i;w))\right| \\ &\leq  \lambda^l \epsilon + L \sqrt{\frac{2K\log 2 + 2\log (1/\delta)}{n}}
\end{aligned}
\]
\end{proof}

\section{Proof of Proposition 2}
The proof of Proposition 2 will closely follow the proof of Proposition 4 and Lemma 5 in \cite{kenjiGeneralization}. Our main technical tool will be controlling the variance in Bernstein-type bounds to obtain an upper bound which has rate $\mathcal{O}(\frac{1}{n})$.  Consider the output of a CNN, given an image as $z$, with abuse of notation (we used $z$ to represent the representation layer, however for the sake of consistency with \cite{kenjiGeneralization} we denote $z$ as the output here). Every activation in the neuron can be written as a sum over the paths between the input layer and the representation as $z_i = \sum_{p} \alpha_p x_p$, where $x_p$ is the input neuron connected to the path and $\alpha_p$ is the weight of the path. One interesting property is the fact that this weight is simply the multiplication of all weights over the path $w_p$ with a binary value. When only max-pooling and ReLU non-linearities are used, that binary value is 1 if all activations are on and 0 if at least one of them is off. This is due to the fact that max-pooling and ReLU either multiply the input with a value of $1$ or $0$. We call this binary variable $\sigma(x,w)$. Hence, each entry is;
\begin{equation}
z_i = \sum_{p} x_{p} \sigma_p(x,w) w_p
\end{equation}
We can note \mbox{$\bar{z} = [x_{\bar{0}} \sigma_p(x,w),\cdots,x_{\bar{P}}\sigma_p(x,w)]$} as a vector with dimension equal to the number of paths. Next, we can explicitly compute the generalization bound over $l_2$ loss as;
\begin{equation}
\begin{aligned}
R(A_s) &= E[l(y,h(x;w))] - \frac{1}{n}\sum_i l(y_i,h(x_i;w) \\
&= E[\|h(x;w)-y\|_2^2] - \frac{1}{n} \sum_i \|h(x_i;w) - y_i\|_2^2 \\
&= \sum_{c} \left( w_c^\intercal \left[ E[zz^\intercal] - \frac{1}{n}\sum_i z_i z_i^\intercal \right] w_c \right) \\
&+ 2 \sum_{c} \left( \left[ \frac{1}{m} \sum_{i} y_{i,k} z_i^\intercal - E[y_k z^\intercal] \right] w_c \right) \\
&+ E[y^\intercal y] - \frac{1}{n} \sum_i y_i^\intercal y_i
\end{aligned}
\label{decompose}
\end{equation}
We will separately bound each term in the following sub-sections. We first need to prove a useful lemma we will use in the following proofs.

\begin{lemma}
\textbf{Matrix Bernstein inequality with variance control (corollary to Theorem 1.4 in \cite{tropp2012})}. Consider a finite sequence $\{M_i\}$ of independent, self-adjoint matrices with dimension d. Assume that each random matrix satisfies $E[M_i]=0$ and $\lambda_{max}(M_i) \leq R$ almost surely. Let $\gamma^2 = \|\sum_i E[M_i^2]\|_2$. Then, for any $\delta >0$, if $t \leq \gamma$; with probability at least $1-\delta$,
\[
\lambda_{max}\left(\sum_i M_i \right) \leq \left(\frac{3\gamma + R}{6}\right) \log \frac{d}{\delta}
\]
\end{lemma}
\begin{proof}
Theorem 1.4 by Tropp \cite{tropp2012} states that for all $t \geq 0$,
\[
P\left(\lambda_{max}\left(\sum_i M_i \right) \geq t \right) \leq d \exp\left(\frac{-t^2/2}{\gamma^2 + Rt/3}\right)
\]
By using the assumption that $t \leq \gamma$,
\[
d \exp\left(\frac{-t^2/2}{\gamma^2 + Rt/3}\right) \leq d \exp\left(\frac{-t/2}{\gamma + R/3}\right)
\]
and substituting $\delta=d \exp\left(\frac{-t/2}{\gamma + R/3}\right)$, the result is implied.
\end{proof}

\paragraph{Bounding $z^\intercal z$ term:}
This will follow directly from the Matrix form of the Bernstein inequality, which is stated as Lemma 1. By using $\xi, M_z$ and $P$ as defined in the main text, Lemma 1 shows that with probability at least $1-\delta$,
\begin{equation}
\lambda_{\text{max}} \left(\left[ E[zz^\intercal] -\frac{1}{n}\sum_i z_i z_i^\intercal \right] \right) \leq (\frac{2 \xi}{n} + \frac{2M_z}{3n}) \log \frac{P}{\delta}
\end{equation}
By using the definition of the Matrix norm and the Cauchy-Schwarz Inequality, one can show that;
\begin{equation}
\begin{aligned}
&\sum_{c} \left( w_c^\intercal \left[ E[zz^\intercal] - \frac{1}{n}\sum_i z_i z_i^\intercal \right] w_c \right) \\
&\quad\leq C \left( \max_c \|w_c\|_2^2 \right)   (\frac{2 \xi}{n} + \frac{2M_z}{3n}) \log \frac{P}{\delta} \\
&\quad \leq  \frac{2 C  M_w (3\xi + M_z)}{3n} \log \frac{P}{\delta}
\end{aligned}
\label{bound1}
\end{equation}

\paragraph{Bounding $z^\intercal y$ term:}
We need to bound $ \sum_{c} \left( \left[ \frac{1}{n} \sum_{i} y_{i,c} z_i^\intercal - E[y_c z^\intercal] \right] w_c \right)$. In order to bound this term, we will first use the fact that $y$ is a 1-hop vector and the fact that $y_cz^\intercal=z^\intercal$ if $y_c=1$, and $0$ otherwise. Hence,
\[
E[y_c z^\intercal] = E[E[y_c z^\intercal|y_c=1]] = E[z^\intercal|y_c=1] \mu_c
\]
where $\mu_c = p(y_c=1)$. Using this fact, we can state;
\begin{equation}
\begin{aligned}
&2 \sum_{c} \left( \left[ \frac{1}{n} \sum_{i} y_{i,k} z_i^\intercal - E[y_k z^\intercal] \right] w_c \right) \\ 
&= 2 \sum_{c} \left( \left[ \frac{1}{n} \sum_{i} y_{i,k} z_i^\intercal - E[z^\intercal|y_c=1]\mu_c \right] w_c \right) \\
&\overset{(a)}{=} 2 \sum_c \frac{|n_c|}{n} \left(\frac{1}{|n_c|} \sum_{i \in n_c} z_i^\intercal - E[z^\intercal|y_c =1] \right) w_c \\
&\quad + 2 \sum_c E[z^\intercal w_c|y_c=1] \left( \frac{|n_c|}{n} - \mu_c \right)
\end{aligned}
\end{equation}
In (a), we noted the training examples from class $c$ as $n_c$. We can use the Bernstein inequality, which states that
\begin{equation}
\begin{aligned}
&p\left(\left[\frac{1}{|n_c|} \sum_{i \in n_c} z_i^\intercal w_c - E[z^\intercal w_c|y_c =1] \right]>t\right) \\
&\leq \exp(-\frac{n_c t^2}{2 \xi^2 + \frac{2t}{3}}) 
\end{aligned}
\end{equation}
where $\xi^2 = \frac{1}{n_c} \sum_i Var\{z_i\}$. Using the assumption that $Var\{z_i\} \leq \delta$, we can state that with probability at least $1 - \delta$,
\begin{equation}
\begin{aligned}
&2 \sum_{c} \left( \left[ \frac{1}{n} \sum_{i} y_{i,k} z_i^\intercal - E[y_k z^\intercal] \right] w_c \right) \\ 
&\leq  \frac{2C (\xi + 1)}{3 n} \log \frac{1}{\delta} + 2 C \epsilon_y
\end{aligned}
\label{bound2}
\end{equation}
Here $\epsilon_y$ is a term which bounds the variance of the class label distribution, which is defined in the next section.

\paragraph{Bounding $y^\intercal y$ term:}
This term is both independent of the learning algorithm and the weights learned and can be simply made to vanish to zero if the number of samples per class directly follows the population densities. Hence, we do not include a specific rate for this quantity and simply denote it with $\epsilon_y$ and assume that it goes to $0$ with a rate better or equivalent to a linear rate. See the main text for a detailed explanation as to why we choose to not include $\epsilon_y$ in the analysis.

After bounding each term in (2), we can now state the proof for Proposition 2.

\begin{proof}
By using the decomposition in (2) and the bounds in (4,7), we can state that
\[
\begin{aligned}
&R(A_s)  \\ &\leq \frac{2C\left(\left(\xi+1\right)\log\frac{1}{\delta} +M_w \left(3\xi+M_z\right) \log\frac{P}{\delta}\right)}{3n} + \left(2C+1\right)\epsilon_y
\end{aligned}
\]
\end{proof}

\section{Derivation of (7, Main Paper)}
In equation (7) of the main paper, we stated that
\begin{equation}
KL(p_w(z|x,x^\star)||p_w(z)) \sim \| \log h^\star(x^\star;w^\star)\|
\end{equation}
In this section, we formally derive this claim using the log-Uniform assumption. In order to compute $KL(p_w(z|x,x^\star)||p_w(z))$, we need to choose a prior distribution for $z$. As discussed in depth in \cite{informationdropout}, the use of ReLU activations empirically suggests that a good choice for this prior would be the log-uniform distribution. Hence, we consider the log-Uniform prior. We first use the definition of the KL-divergence as;
\begin{equation}
\begin{aligned}
&KL(p_w(z|x,x^\star)||p_w(z))  \\
&= - E_{p_w(z|x,x^\star)}[\log p_w(z)] + E_{p_w(z|x,x^\star)}[\log p_w(z|x,x^\star) ] \\
\end{aligned}
\end{equation}
Since we know the distribution of $p_w(z|x,x^\star)$ as $\mathcal{N}(h^o(x,w^o),h^\star(x^\star,w^\star))$, and using the assumption that the covariance matrix is diagonal,
\begin{equation}
E_{p_w(z|x,x^\star)}[\log p_w(z|x,x^\star) ] = \left\|\frac{1}{2}(1 + \log 2\pi h^\star(x^\star;w^\star))\right\|_1
\end{equation}
If we use the log-uniform prior, the first term in the KL-divergence can be computed as;
\begin{equation}
E_{p_w(z|x,x^\star)}[\log p_w(z)]  = E_{p_w(z|x,x^\star)}[c_1 + c_2 z] = c 
\end{equation}
where we use the fact that the logarithm of the pdf of a log-uniform distribution is $c_1 + c_2z$ with appropriate constants. Furthermore, the norm of $h^0(x,w)$ does not affect the output as it is followed with a soft-max operation, which is invariant up-to a scalar multiplication. Hence, we can safely consider its norm to be a constant $c$. Using both terms,
\begin{equation}
KL(p_w(z|x,x^\star)||p_w(z)) = \bar{c}\log h^\star(x^\star;w^\star)\| - c
\end{equation}
with appropriate constants $\bar{c}$ and $c$. We do not include $c$ in the optimization since an additional constant does not change the result of the optimization and we include $\bar{c}$ in the trade-of parameter $\beta$.

\section{Additional Results}
In this section, we provide two experimental results missing in the paper: first, an analysis of accuracy vs. the amount of $x^\star$ provided for image classification and second, a comparison of our method with baselines for the task of ImageNet image classification with $200$K images. We also provide further qualitative analysis of the relationship between variance control and our method.

\paragraph{Accuracy vs. Partial $x^\star$ for Image Classification:} In the main text, we already studied the case where only a partial $x^\star$ is available and showed that even a small percentage of $x^\star$ is enough for multi-modal machine translation experiments. Due to the limited space, we provide the same experiment for the image classification here in Figure~\ref{fig:partialxstar} and show that as long as a small percentage of the dataset has privileged information, our algorithm is effective.

\begin{figure}[ht]
\includegraphics[width=\columnwidth]{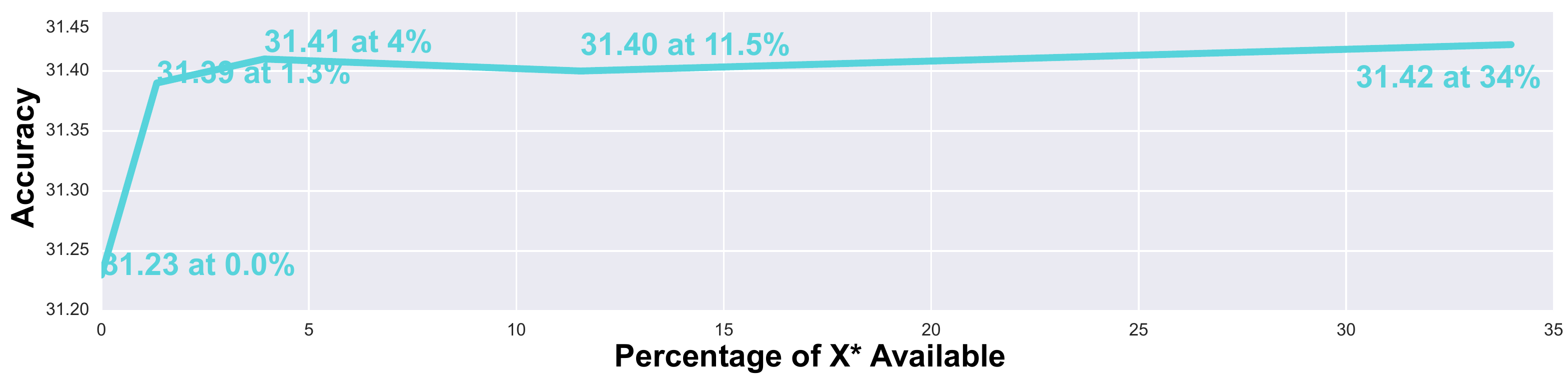}
\caption{Accuracy vs amount of privileged information ($x^\star$) available for image classification using $75$K ImageNet images. We plot top-1, single crop accuracy.}
\label{fig:partialxstar}
\end{figure}

\paragraph{ImageNet with 200K Images}
In the main paper, we performed the ImageNet image classification experiment with only 75K training images (a mid-sized dataset) and showed that our method learns significantly faster (by reaching higher accuracy) than all baselines. One might ask, \emph{would the result still hold if we had up to 200K images (a larger-sized dataset)?} In order to answer this question, we carry out further experiments and show the results in Table~\ref{tab:75kbaselines_imagenet_1000_classification}; the results suggest that our method matches the performance of the best baselines in the 200K case. Hence, we can conclude that our method, \textit{i.e.} marginalization, provides no harm in a large-scale dataset regime.

\begin{table}[htb]
\caption{ We compare our method's performance with several baselines. We train with \textbf{200} images per each of the 1000 ImageNet classes, leaving us with  200 $\times 10^3$ images in total. Since we utilize only 600K (or less) of the 1.28M images from the CLS-LOC ImageNet dataset across all of our experiments, we use a randomly selected subset of the remaining 628K images as a hold-out set for evaluation. Accuracy is given in \%, from 0 to 100. Multi-crop accuracy is computed not via individual voting on the correct class by each crop, but rather by taking an \texttt{arg max} over classes after summing the \texttt{softmax} score vectors of each individual crop.}
\label{tab:75kbaselines_imagenet_1000_classification}
\centering
\resizebox{\columnwidth}{!}{%
\begin{tabular}{lcccc}
\toprule
 & \multicolumn{2}{c}{Single Crop} &  \multicolumn{2}{c}{Multi-Crop}\\
\textbf{Model} & top-1 & top-5 & top-1 &top-5 \\
\midrule
No-$x^\star$ \cite{Simonyan2014} & 55.99 & 79.21 & 58.60 & 80.98 \\
\midrule
GoCNN \cite{GroupOrthogonalCNN} & 50.73 & 75.39& 53.37 & 77.61\\
Modal. Hallucination \cite{Hoffman_2016_CVPR_ModalityHallucination}& 52.28 & 76.33 & 55.66 & 78.79 \\
Info. Dropout \cite{informationdropout} & 54.60 & 77.89 & 58.47 & 81.25 \\
Our LUPI  & $55.20$ & $78.72$ & $58.17$ & $80.90$ \\
Gaussian Dropout \cite{JMLR_Srivastava14a_Dropout} & 55.48 & 78.87 & 58.11 &  80.58 \\
MIML-FCN \cite{MIMLFCNCVPR2017}/ResNet-50 & 56.00 & 78.83 & 59.14 & 81.05 \\
Multi-Task w/ Mask \cite{LearningDeconvolutionNetwork2015}  & 56.22  & 79.56 & 59.39 & $\mathbf{81.68}$ \\
MIML-FCN \cite{MIMLFCNCVPR2017}/VGG& 56.23 & 79.51  & 58.85 & 81.23 \\
Multi-Task w/ Bbox & $\mathbf{56.32}$ &  $\mathbf{79.45}$ & $\mathbf{59.29}$ & 81.48 \\
\bottomrule
\end{tabular}}
\end{table}

\section{Additional Qualitative Analysis of the Method}
In Figure (\ref{fig:qual_var_heatmap_only}), we visualize the computed variance of the heteroscedastic dropout.
\begin{figure}[ht]
\includegraphics[width=\columnwidth]{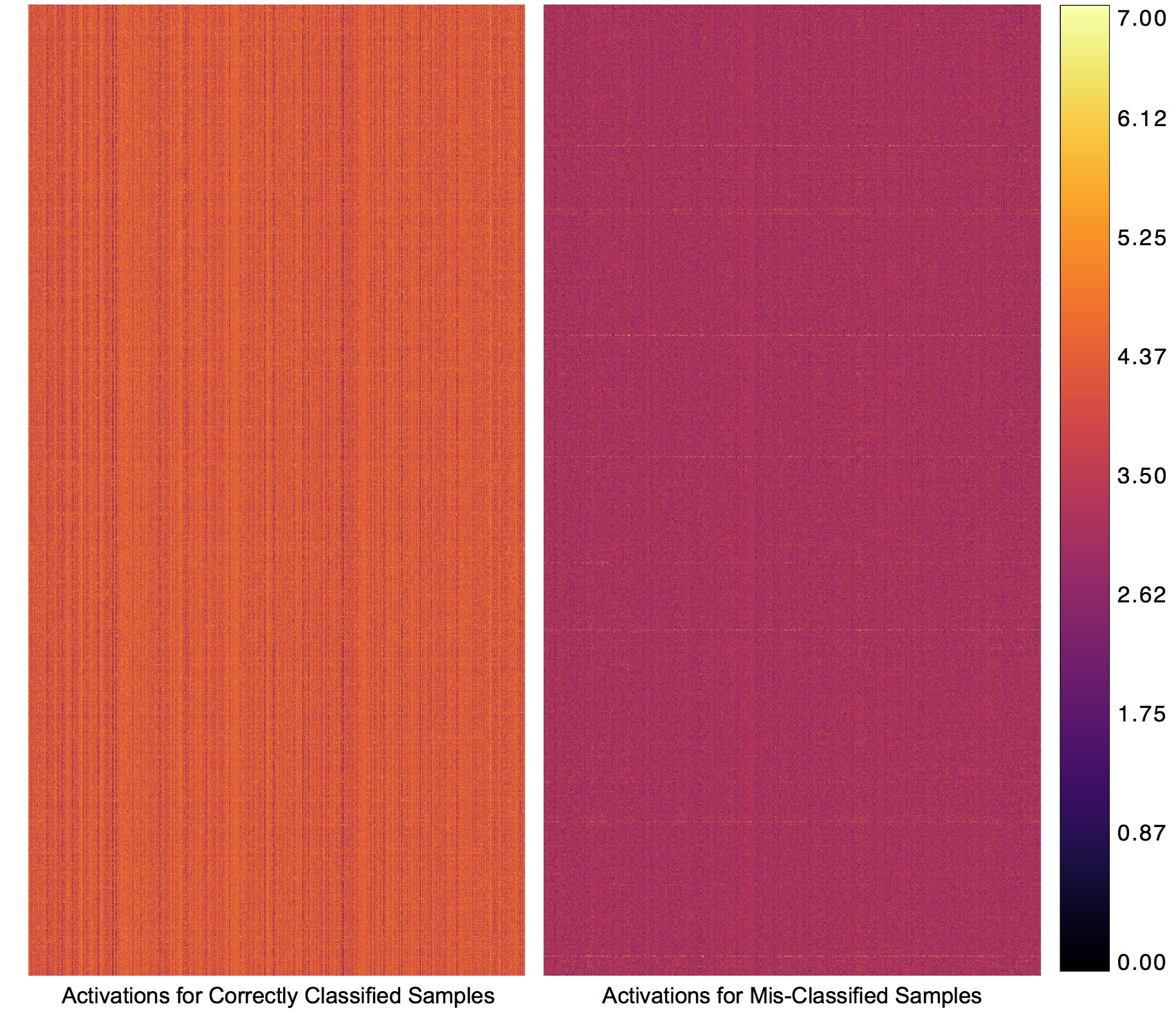}
\caption{Visualization of the computed variance of our heteroscedastic dropout for 8000 random samples from the validation set that our algorithm mis-classifies, as well as 8000 random samples it correctly classifies. The plot is a heatmap of activations, with dimensions (\texttt{num\_images}$\times$\texttt{num\_channels}).}
\label{fig:qual_var_heatmap_only}
\end{figure}
Figure (\ref{fig:qual_var_heatmap_only}) supports our hypothesis that our algorithm controls the variance since mis-classified examples are expected to have high variance/uncertainty and need to be multiplied with a low value to be controlled. The visualization is fairly uniform, especially for misclassified examples, but we believe the $h^{\star}$ has interesting information in it which can be further utilized in applications like confidence estimation and is an interesting future work direction.
\section{Additional Implementation Details}
In this section, we give all of the implementation details of our algorithm, as well as the implementation details of the baselines we used in our experimental study. In order to ensure full reproducibility of all experiments, we share our source code \footnote{https://github.com/johnwlambert/dlupi-heteroscedastic-dropout}.  We found that in all experiments that could converge, from training set sizes of $32K$ up to $600K$ images, an adaptive $10x$ learning rate decay schedule significantly outperforms the traditional 30-epoch fixed $10x$ learning rate decay schedule. We consistently observe performance gains of $5\ - 10\%$ with the learning rate schedule set adaptively according to whether or not performance on the hold-out validation set has reached a plateau. Unless otherwise noted, we utilize SGD with momentum set to 0.9 for all models, and a learning rate schedule that starts at $1 \times 10^{-2}$, as \cite{Simonyan2014} suggests.

\paragraph{Heteroscedastic Dropout Implementation:}
We set $\lambda=100$ in all experiments, although we found this was not a meaningful hyperparameter. We found the training to be prone to convergence in local optima and restarted training if the distribution over class logits was still uniform after 30 epochs. We use a weight decay of $1 \times 10^{-4}$ in all experiments, ADAM, and a learning rate of $1 \times 10^{-3}$, as described in Section 3.2 of the paper. We cropped images to a standard size of $224 \times 224$ before feeding them into the network. 

We scale the batch size $m$ with respect to the size of the training set. For example, for the $75$K model, we use a batch size of 64. For the $200$K Model, we use a batchsize of 128. For the $600$K model, we utilize curriculum learning and a batch size of 256. We first train the fc layers in the $x^{\star}$ tower for 8 epochs with ADAM, a batch size of 128, and a learning rate $1 \times 10^{-3}$, and then fix the $x^{\star}$ fc weights and fine-tune the fc layers of the $x$ tower with the ADAM optimizer and a learning rate of $1 \times 10^{-7}$ and a batch size of 256.

\paragraph{No-$x^\star$:}
A baseline model without access to any privileged information. We use a batch size of 256.

\paragraph{Gaussian Dropout \cite{JMLR_Srivastava14a_Dropout}:} 
We draw noise from  \mbox{$\mathcal{N}\big(\mathbf{1},diag(\mathbf{1})\big)$} because the authors of \cite{JMLR_Srivastava14a_Dropout} state that $\sigma$ should be set to $\sqrt{ \frac{(1- \mbox{drop prob})}{ \mbox{drop prob} )}}$.  We did not include a regularization loss on the covariance matrices of the random noise. We use SGD with momentum set to 0.9, a learning rate of $1 \times 10^{-2}$, and a batch size of 256.

\paragraph{Multi-Task with Bbox:}
We add one extra head to the VGG network that, just as the classification head, accepts pool5 activations. This regression head produces the center coordinates ($x_{\mathrm{cent}},y_{\mathrm{cent}}$) and width and height of a bounding box, all normalized to $[0,1]$. As our loss function, we use a weighted sum of cross entropy loss and $\lambda=0.1$ times the bounding box regression loss. We use a batch size of 200 instead of 256 because of GPU RAM constraints of $\sim 64$ GB.

\paragraph{Multi-Task with Mask:}
In order to predict pixel-wise probabilities between a background and foreground (object) class, we require an auto-encoder network that can preserve spatial information. We experiment with two architectures (DeconvNet) \cite{DCGANRadford2015} \cite{LearningDeconvolutionNetwork2015}. We chose the DeconvNet architecture for its superior performance, which we attribute to its far greater representation power than DCGAN (the DeconvNet architecture utilizes 15 convolutions instead of the much shallower 5 convolution architecture of the DCGAN generator/discriminator, versus 13 conv. layers in VGG)\cite{LearningDeconvolutionNetwork2015}\cite{DCGANRadford2015} \cite{Simonyan2014}. 
As our loss function, we use a weighted sum of cross-entropy losses over classes and $\lambda=0.1$ times the cross entropy loss over masks . We use a batch size of 128 instead of 256 because of GPU RAM constraints of $\sim 64$ GB.

\paragraph{GoCNN \cite{GroupOrthogonalCNN}}
We found that the models could not converge when the suppression loss (computed as the Frobenius norm of the masked activations) is multiplied only by (1/32), as the authors utilize in their work. We found that the model could learn if the suppression loss was multiplied by (1/320) or (1/3200) with ADAM, a learning rate of $1 \times 10^{-3}$, and a batch size of 256. We use a black and white (BW) mask for $x^{\star}$.

\paragraph{Information Dropout \cite{informationdropout}} As we note in the main paper, we found a VGG-16 network with two Information Dropout layers, each succeeding one of the first two fully connected layers, could only converge with a sigmoid nonlinearity in the fc layers. We keep the ReLU nonlinearity in the convolutional layers. We train with a batch size of 128, set $\beta=3.0$, set $\alpha_{\text{maximum}} = 0.3$, sample from a log-normal distribution (by exponentiating samples from a normal distribution), and employ  an improper log-uniform distribution as our prior, as the authors used for their CIFAR experiments.

\paragraph{MIML-FCN \cite{MIMLFCNCVPR2017}:}
We compare the use of a VGG-16 or ResNet-50 architecture, with a batch size of 256 and $\lambda=1 \times 10^{-8}$, which we tuned manually by cross-validation. For the ResNet-50 architecture, we start the learning rate schedule at $1 \times 10^{-1}$.  We share the convolutional layer parameters across both parameters, and thus find far superior performance when $x^{\star}$ is provided as an RGB mask, rather than a black and white (BW) mask, because the privileged information is more closely aligned with the input $x$.

\paragraph{Modality Hallucination \cite{Hoffman_2016_CVPR_ModalityHallucination}:}
Due to the memory requirements of 3 VGG towers with independent parameters, we chose to share the feature representation in the convolutional layers and to incorporate the hallucination loss between the fc1 activations of the depth and hallucination networks. We use a batch size of 128. For identical reasons as those stated in the previous paragraph, RGB masks are a superior representation for $x^{\star}$ than BW masks for this model.
\end{document}